\documentclass[10pt,twocolumn]{article}

\usepackage{amsmath,amssymb,amsthm}
\usepackage[dvips]{graphicx}          
\usepackage{algorithm}
\usepackage{algpseudocode}
\usepackage{booktabs}
\usepackage{multirow}
\usepackage{url}
\usepackage{hyperref}
\usepackage{caption}
\usepackage{subcaption}
\usepackage{cite}
\usepackage{mathptmx}
\usepackage{geometry}
\usepackage{listings}
\usepackage{xcolor}
\newtheorem{definition}{Definition}
\newtheorem{lemma}{Lemma}

\title{Intelligence as Trajectory-Dominant Pareto Optimization}
\author{
    Truong Xuan Khanh\textsuperscript{1} \and Truong Quynh Hoa\textsuperscript{1} \\
    \textsuperscript{1}H\&K Research Studio, Clevix LLC, Hanoi, Vietnam
}
\date{\today}

\begin{document}

\maketitle

\begin{abstract}
Recent advances in artificial intelligence have produced systems with impressive task-level performance, yet many display an early saturation in long-horizon adaptability despite continued optimization. Standard evaluation paradigms—centered on point-wise performance, terminal outcomes, or static multi-objective trade-offs—struggle to explain why systems that remain locally optimal nevertheless become developmentally constrained over time.

In this work, we argue that this limitation arises from a deeper mischaracterization of intelligence itself. We propose a foundational reformulation in which intelligence is understood as a property of agent trajectories rather than isolated states. Within this view, we introduce \textit{trajectory-dominant Pareto optimization}, a path-wise generalization of classical Pareto optimality where dominance relations are defined over accumulated multi-objective costs along entire trajectories. This formulation reveals the existence of \textit{Pareto traps}: regions of trajectory space that are locally non-dominated yet globally restrictive, preventing access to alternative developmental paths without temporary degradation.

To clarify the structure of these constraints, we develop a mathematical taxonomy of Pareto traps based on their geometric signatures, and introduce the \textit{Trap Escape Difficulty Index (TEDI)}, an intrinsic measure of how difficult it is for an agent to escape a given trap. Dynamic intelligence ceilings emerge naturally as geometric consequences of trajectory-level dominance, with TEDI providing a principled way to compare the rigidity of such ceilings across systems and optimization regimes.

We illustrate these ideas using a minimal model contrasting point-wise and trajectory-dominant agents, showing how ceiling formation and escape arise independently of architectural limits or computational scale. By shifting attention from local performance to trajectory accessibility, this framework offers a structural explanation for long-horizon stagnation and establishes a foundation for studying adaptive restructuring, developmental phase transitions, and commitment-aware intelligent systems.
\end{abstract}

\section{Introduction}
Artificial intelligence systems have achieved remarkable levels of competence across a wide range of tasks, driven by advances in learning algorithms, optimization techniques, and computational scale \cite{russell2010artificial,lecun2015deep}. Yet alongside these successes, a persistent pattern has emerged: many systems exhibit an early saturation in long-horizon adaptability. Continued optimization deepens performance within familiar regimes, but fails to expand the range of problems or structures the system can reliably engage with over time.

This phenomenon is often interpreted as a limitation of data, model capacity, or training methodology. However, such explanations remain incomplete. In practice, systems may continue to improve efficiency, consistency, or reward accumulation while becoming increasingly constrained in their developmental trajectories. The result is not an abrupt failure of learning, but a form of structural stagnation, in which progress persists locally while global adaptability plateaus \cite{hernandez2017measure,sutton2019bitter}.

Existing evaluation paradigms offer limited insight into this behavior. Standard benchmarks emphasize task-level performance, terminal outcomes, or aggregate rewards, implicitly treating intelligence as a point-wise or episodic property \cite{hernandez2017measure}. Even in multi-objective reinforcement learning, dominance and trade-offs are typically assessed at the level of outcomes or policies rather than over the full trajectories that give rise to them \cite{vamplew2022scalar,parisi2016multiobjective}. As a consequence, these approaches struggle to explain why systems that remain locally optimal nevertheless become unable to access qualitatively new solution regimes.

In this work, we argue that the source of this limitation lies deeper than any particular algorithm or architecture. Rather than viewing intelligence as a property of isolated states or terminal performance, we propose that intelligence should be understood as a property of \textit{trajectories}: extended sequences of decisions evolving under competing objectives within a structured space. From this perspective, what constrains long-horizon intelligence is not merely what an agent achieves, but which trajectories remain accessible under its optimization rules.

To formalize this view, we introduce \textit{trajectory-dominant Pareto optimization}, a path-wise generalization of classical Pareto optimality in which dominance relations are defined over accumulated multi-objective costs along entire trajectories. This formulation exposes a class of geometric constraints that do not appear under point-wise evaluation. In particular, it reveals the existence of \textit{Pareto traps}: regions of trajectory space that are locally non-dominated, yet globally restrictive, in the sense that reaching superior developmental paths requires temporary degradation in at least one objective.

Within such regions, optimization continues, but only within a confined subset of trajectories. We refer to the resulting limitation as a \textit{dynamic intelligence ceiling}: the maximal effective intelligence attainable while an agent remains confined to a given Pareto trap. These ceilings are not absolute capability limits, nor do they arise from insufficient learning or computational resources. Instead, they emerge as geometric consequences of locally conservative dominance criteria applied over trajectories, complementing recent observations of stagnation and collapse-like dynamics in adaptive systems \cite{shumailov2023curse,khanh2026dynamic,khanh2025entropy}.

Building on this structural account, we further show that Pareto traps are not uniform in their severity. Different traps impose qualitatively different constraints on escape, depending on their geometric structure and the behavioral rigidity they induce. To capture this variation, we introduce a mathematical taxonomy of Pareto traps and propose the \textit{Trap Escape Difficulty Index (TEDI)}, an intrinsic quantitative measure of how difficult it is for an agent to escape a given trap. TEDI provides a principled basis for comparing the rigidity of dynamic intelligence ceilings across systems and optimization regimes.

The contribution of this work is conceptual rather than algorithmic. We do not propose new learning rules, architectures, or benchmarks. Instead, we offer a definition-level reformulation that explains why developmental stagnation can coexist with continued optimization, and why overcoming intelligence ceilings requires changes in trajectory dominance rather than further pointwise improvement. A minimal illustrative model is used to demonstrate how these effects arise even in simple environments, independent of scale or architectural complexity.

By reframing intelligence as a geometric property of trajectory accessibility, this work aims to clarify the structural origins of long-horizon limits in artificial systems and to lay the groundwork for future research on adaptive restructuring, developmental phase transitions, and commitment-aware forms of artificial intelligence, including open-ended and novelty-driven paradigms \cite{lehman2011abandoning,mouret2015illuminating,wang2020enhanced}.

\section{Related Work}
The present work draws on and connects several strands of research concerning the evaluation of intelligence, multi-objective optimization, and long-horizon developmental dynamics. Rather than extending any single line of work, our aim is to clarify a structural limitation shared across these domains by reformulating intelligence as a trajectory-level optimization problem.

\subsection{Evaluation and Measurement of Intelligence}
A long-standing challenge in artificial intelligence concerns how intelligence should be defined and evaluated. Classical approaches emphasize task-level performance, benchmark scores, or aggregate reward accumulation, implicitly treating intelligence as a static or point-wise quantity \cite{russell2010artificial}. More recent work has sought to broaden this view by considering generality, adaptability, and performance across diverse task distributions \cite{hernandez2017measure}.

While these frameworks provide valuable descriptive metrics, they remain limited in their ability to explain why systems that continue to improve locally may nevertheless exhibit early saturation in long-horizon adaptability. In particular, they do not distinguish between improvements that refine behavior within an existing developmental regime and those that expand the set of regimes the system can access over time.

\subsection{Multi-Objective Optimization and Pareto Efficiency}
Multi-objective optimization provides a natural language for reasoning about trade-offs between competing objectives, and has been widely studied in reinforcement learning and control \cite{vanmoffaert2014multiobjective,parisi2016multiobjective}. In these settings, Pareto optimality is typically defined over terminal outcomes, episodic returns, or policy-level performance, with learning algorithms seeking to approximate a Pareto front over achievable solutions.

Recent work has emphasized that scalar reward formulations are insufficient to capture the structure of such trade-offs, motivating explicit Pareto-based learning objectives \cite{vamplew2022scalar,liu2025pareto}. However, these approaches continue to evaluate dominance at the level of outcomes or policies, rather than over the full trajectories that generate them. As a result, they cannot capture situations in which locally Pareto-efficient behavior at the trajectory level restricts access to globally superior developmental paths.

\subsection{Exploration, Novelty, and Open-Ended Learning}
A complementary body of work has focused on overcoming stagnation through exploration, novelty, and open-ended learning. Approaches such as novelty search \cite{lehman2011abandoning}, quality-diversity methods \cite{mouret2015illuminating}, and POET-style environment co-evolution \cite{wang2020enhanced} demonstrate that relaxing objective-driven optimization can yield more diverse and adaptable behaviors.

While effective in practice, these methods are primarily algorithmic. They do not offer a formal account of why stagnation arises in the first place, nor do they characterize the structural conditions under which exploration must violate local optimality in order to achieve long-horizon gains. In contrast, the present work provides a definition-level explanation for such phenomena by identifying geometric constraints induced by trajectory-level dominance.

\subsection{Stagnation, Collapse, and Developmental Ceilings}
Recent analyses have highlighted failure modes in which intelligent systems lose adaptability as optimization proceeds, including collapse-like dynamics under recursive training and reinforcement \cite{shumailov2023curse}. Related theoretical work has argued that intelligence limits often manifest not as hard capability bounds, but as ceilings that stabilize prematurely under continued optimization \cite{khanh2026dynamic,khanh2025entropy}.

Our contribution complements these perspectives by identifying a common structural origin for such ceilings. Rather than attributing stagnation to data scarcity, scale limitations, or misaligned objectives, we show that dynamic intelligence ceilings arise as geometric consequences of locally conservative dominance relations over trajectories.

In summary, existing work has documented stagnation, explored algorithmic remedies, and proposed descriptive metrics. The present work differs by offering a unifying structural account: intelligence ceilings emerge when optimization is confined to locally non-dominated regions of trajectory space, even under continued learning and improvement.

\section{Trajectory-Dominant Pareto Optimization}
We formalize intelligence as a property of trajectories generated by an agent operating under multiple competing objectives. This section introduces trajectory-dominant Pareto optimization as a path-wise generalization of classical Pareto optimality, and establishes the foundational properties required for analyzing long-horizon developmental constraints.

\subsection{Trajectories and Accumulated Costs}
We consider a finite-horizon Markov Decision Process (MDP) with finite state space \(\mathcal{X}\), action space \(\mathcal{A}\), transition kernel \(P\), initial state \(x_0 \in \mathcal{X}\), and horizon \(T < \infty\). A trajectory is defined as
\[
\tau = (x_0,a_0,x_1,a_1,\ldots ,x_T),
\]
where \(x_{t+1} \sim P(\cdot | x_t,a_t)\).

Let \(\ell_i(x_t,a_t,t)\) denote the instantaneous cost associated with objective \(i \in \{1,\ldots,m\}\). The accumulated cost of objective \(i\) along trajectory \(\tau\) is
\[
J_i(\tau) = \sum_{t=0}^{T-1} \ell_i(x_t,a_t,t),
\]
and the trajectory cost vector is \(J(\tau) = (J_1(\tau),\ldots,J_m(\tau)) \in \mathbb{R}^m\).

\subsection{Trajectory-Level Dominance}
Classical Pareto optimality defines dominance over outcomes or policy returns. We extend this notion to trajectories.

\begin{definition}[Trajectory Dominance]
A trajectory \(\tau\) is said to \textit{trajectory-dominate} another trajectory \(\tau'\) (denoted \(\tau \succ \tau'\)) if
\[
J_i(\tau) \leq J_i(\tau') \quad \forall i \in \{1,\ldots,m\},
\]
and there exists at least one index \(j\) such that \(J_j(\tau) < J_j(\tau')\).
\end{definition}

\begin{definition}[Trajectory-Pareto Optimality]
A trajectory \(\tau^*\) is \textit{trajectory-Pareto optimal} if there exists no feasible trajectory \(\tau\) such that \(\tau \succ \tau^*\).
\end{definition}

This definition induces a partial order over the finite set of feasible trajectories. By standard arguments based on finiteness and minimal elements in partially ordered sets, the set of trajectory-Pareto optimal trajectories is non-empty.

\subsection{Trajectory-Dominant Optimization}
We say that an agent operates under \textit{trajectory-dominant Pareto optimization} if its decision process avoids selecting trajectories that are trajectory-dominated, even when such trajectories may offer superior short-term or terminal performance under point-wise evaluation.

\subsection{Pareto Traps and Dynamic Intelligence Ceilings}
Trajectory-level dominance permits the emergence of locally non-dominated regions that are globally restrictive. We refer to such regions as \textit{Pareto traps}.

\begin{definition}[Pareto Trap]
A \textit{Pareto trap} is a non-empty subset of feasible trajectories that is locally trajectory-Pareto optimal under small perturbations, yet globally dominated by trajectories that cannot be reached without temporary degradation in at least one objective.
\end{definition}

When an agent remains confined to a Pareto trap, optimization continues but only within a restricted subset of trajectories. We define the resulting limitation as a \textit{dynamic intelligence ceiling}: the maximal effective intelligence attainable while the agent's behavior remains confined to that trap.

Crucially, these ceilings are not imposed by architectural capacity, computational scale, or learning inefficiency. Instead, they arise as geometric consequences of locally conservative dominance relations over trajectories. Escaping a Pareto trap necessarily requires violating local dominance, accepting temporary degradation in one or more objectives in order to access globally superior developmental paths.

This formulation provides the foundation for the taxonomy of Pareto traps and the quantitative analysis of escape difficulty introduced in the following section.

\section{Pareto Traps and Dynamic Intelligence Ceilings}
Trajectory-level dominance introduces a class of structural constraints that do not arise under pointwise or outcome-based evaluation. In this section, we formalize these constraints as Pareto traps, show how they give rise to dynamic intelligence ceilings, and introduce a quantitative framework for characterizing their severity.

\subsection{Pareto Traps as Geometric Structures in Trajectory Space}
Under trajectory-dominant Pareto optimization, dominance relations are defined over accumulated multi-objective costs along entire trajectories. This induces a geometric structure over the space of feasible trajectories, in which locally non-dominated regions may nevertheless be globally restrictive.

\begin{definition}[Pareto Trap (Extended)]
A Pareto trap is a non-empty subset \(S \subseteq \mathcal{T}\) of feasible trajectories that satisfies the following conditions:
\begin{enumerate}
    \item \textit{Local Pareto optimality}: Each trajectory \(\tau \in S\) is locally trajectory-Pareto optimal under small perturbations.
    \item \textit{Global restrictiveness}: There exist trajectories \(\hat{\tau} \notin S\) such that \(\hat{\tau} \succ \tau\) for some \(\tau \in S\), and reaching \(\hat{\tau}\) requires temporary increases in at least one accumulated cost.
\end{enumerate}
\end{definition}

Intuitively, a Pareto trap is not a failure of optimization, but a consequence of optimization that is locally conservative with respect to trajectory-level dominance. Within such a region, agents continue to refine behavior and improve efficiency along admissible directions, yet remain structurally unable to access alternative developmental paths.

This mechanism differs fundamentally from classical local optima in scalar optimization. Pareto traps arise from the geometry of multi-objective dominance over trajectories, rather than from smoothness or convexity properties of a single objective function.

\subsection{Dynamic Intelligence Ceilings}
When an agent remains confined to a Pareto trap, its developmental progress is bounded, even under continued optimization. We formalize this limitation as a dynamic intelligence ceiling.

\begin{definition}[Dynamic Intelligence Ceiling]
Let \(S\) denote a Pareto trap and let \(f:\mathbb{R}^{m}\to \mathbb{R}\) be a monotonic scalarization representing effective intelligence or developmental reach. The \textit{dynamic intelligence ceiling} associated with \(S\) is defined as
\[
C(S) = \sup_{\tau \in S} f(J(\tau)).
\]
\end{definition}

As long as the agent's behavior remains confined to \(S\), further optimization can at most approach \(C(S)\), but cannot exceed it. Importantly, this ceiling is not imposed by architectural capacity, computational resources, or data availability. Instead, it emerges as a geometric consequence of trajectory-level dominance constraints.

This perspective complements recent observations that intelligent systems may exhibit stagnation or collapse-like dynamics despite ongoing learning and reinforcement \cite{shumailov2023curse,khanh2026dynamic}. In our framework, such phenomena arise naturally when optimization remains confined to locally non-dominated trajectory regions.

\subsection{A Mathematical Taxonomy of Pareto Traps}
Not all Pareto traps impose equivalent constraints. Their severity and escape dynamics depend on the geometric structure of the corresponding region in trajectory space. Based on these properties, we identify four fundamental types of Pareto traps.

\begin{itemize}
    \item \textbf{Local Basins}: Compact regions with high curvature and dense clustering of trajectories. Any local deviation increases at least one accumulated cost, making escape require substantial temporary degradation.
    
    \item \textbf{Narrow Corridors}: Regions with a single narrow exit path characterized by steep cost gradients. Escape is possible but highly sensitive to precise trajectory-level perturbations.
    
    \item \textbf{Optimality Plateaus}: Broad regions with low gradients and many mutually non-dominated trajectories. Agents stagnate not due to steep barriers, but due to excessive local optimality.
    
    \item \textbf{Attractor Loops}: Cyclic trajectory patterns reinforced by low behavioral entropy, where agents repeatedly revisit suboptimal behaviors through positive feedback.
\end{itemize}

This taxonomy provides a structural language for comparing developmental constraints across systems, tasks, and optimization regimes. Unlike algorithm-specific classifications, these trap types are defined purely by the geometry of trajectory dominance.

\subsection{Quantifying Escape Difficulty: The Trap Escape Difficulty Index}
While the taxonomy distinguishes qualitative trap structures, a quantitative measure is required to compare their severity. We therefore introduce the \textit{Trap Escape Difficulty Index (TEDI)}, an intrinsic metric of how difficult it is for an agent to escape a given Pareto trap.

\begin{definition}[Trap Escape Difficulty Index]
The Trap Escape Difficulty Index is defined as
\[
\mathrm{TEDI} = \alpha D + \beta S + \gamma B, \qquad \alpha + \beta + \gamma = 1,
\]
where:
\begin{itemize}
    \item \(D\) (\textit{Escape Distance}) measures the normalized distance in trajectory-cost space from the trap to globally dominating trajectories;
    \item \(S\) (\textit{Structural Constraint}) captures geometric rigidity, including curvature, boundary complexity, and connectivity of the trap region;
    \item \(B\) (\textit{Behavioral Inertia}) reflects policy-level rigidity, incorporating action entropy, exploration rate, and adaptability.
\end{itemize}
\end{definition}

TEDI values range from 0 (trivial to escape) to 1 (practically inescapable). Unlike regret, reward, or exploration bonuses, TEDI is a geometric property induced by dominance relations over trajectories. It characterizes the rigidity of a dynamic intelligence ceiling independently of any specific learning algorithm.

\subsection{Illustration in a Minimal Model}

\begin{figure*}[t]
    \centering
    \includegraphics[width=\textwidth]{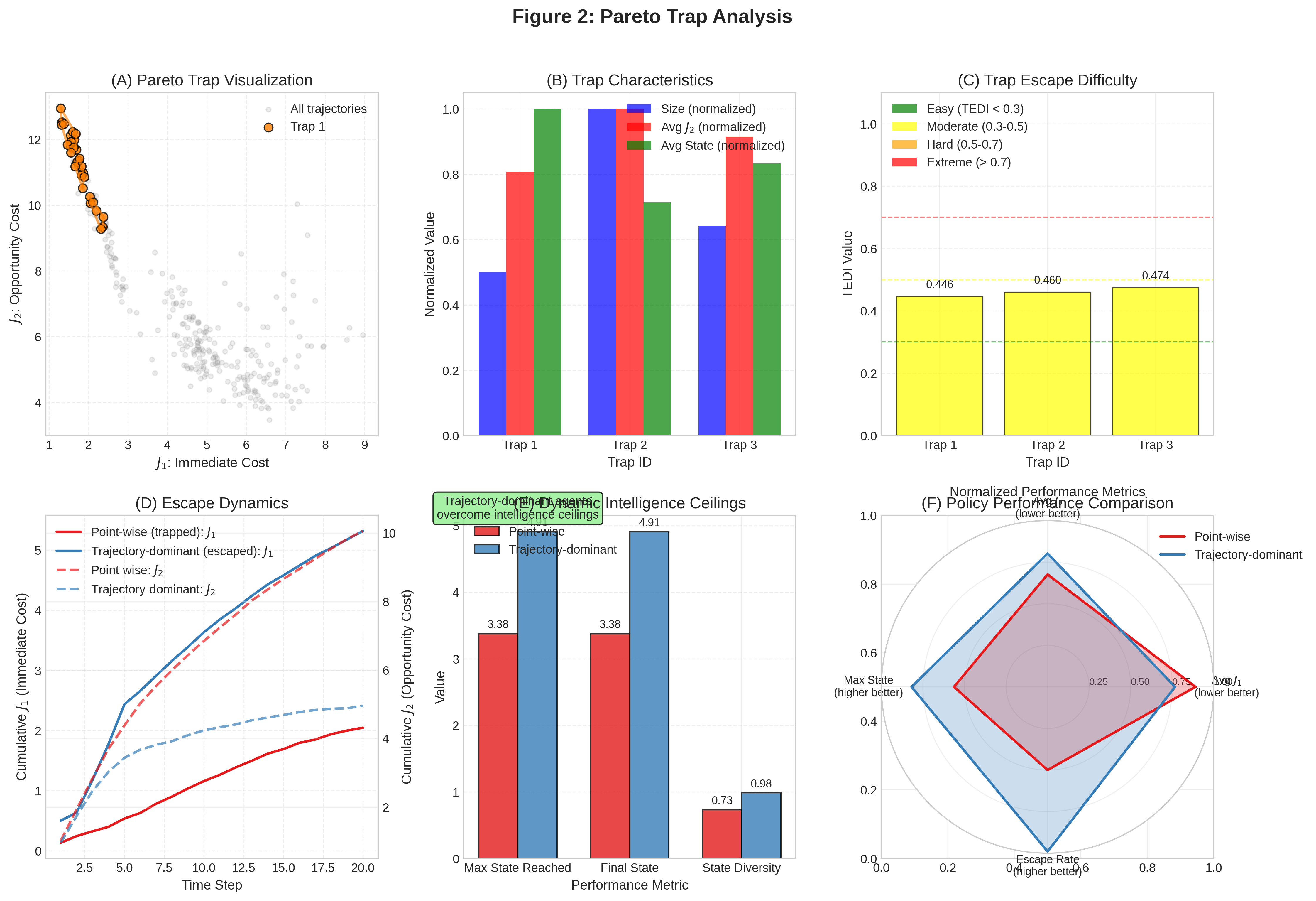}
    \caption{Pareto Trap Analysis. (A) Visualization of trajectories and Trap 1. (B) Trap characteristics across multiple traps. (C) TEDI values and escape difficulty categories. (D) Escape dynamics comparison. (E) Performance metrics showing trajectory-dominant overcoming ceilings. (F) Radar plot of normalized policy performance.}
    \label{fig:pareto_trap_analysis}
\end{figure*}

Figure \ref{fig:pareto_trap_analysis} illustrates the geometric structure of Pareto traps and their associated TEDI values in the minimal illustrative model introduced in the following section. The figure shows that point-wise optimizing agents become confined to high-TEDI traps, while trajectory-dominant agents traverse regions of lower escape difficulty by permitting temporary degradation.

These results highlight a central implication of the framework: overcoming dynamic intelligence ceilings requires violating local dominance constraints, not merely increasing optimization effort within a fixed regime.

The taxonomy and TEDI together provide a principled framework for analyzing when and why intelligent systems stagnate, and for distinguishing ceilings that can be overcome through modest restructuring from those that require fundamental changes in trajectory dominance.

\section{A Minimal Illustrative Model}
To ground the abstract geometry of trajectory-dominant Pareto optimization, we introduce a minimal model that instantiates the emergence of Pareto traps and dynamic intelligence ceilings under identical dynamics.

\begin{figure*}[t]
    \centering
    \includegraphics[width=\textwidth]{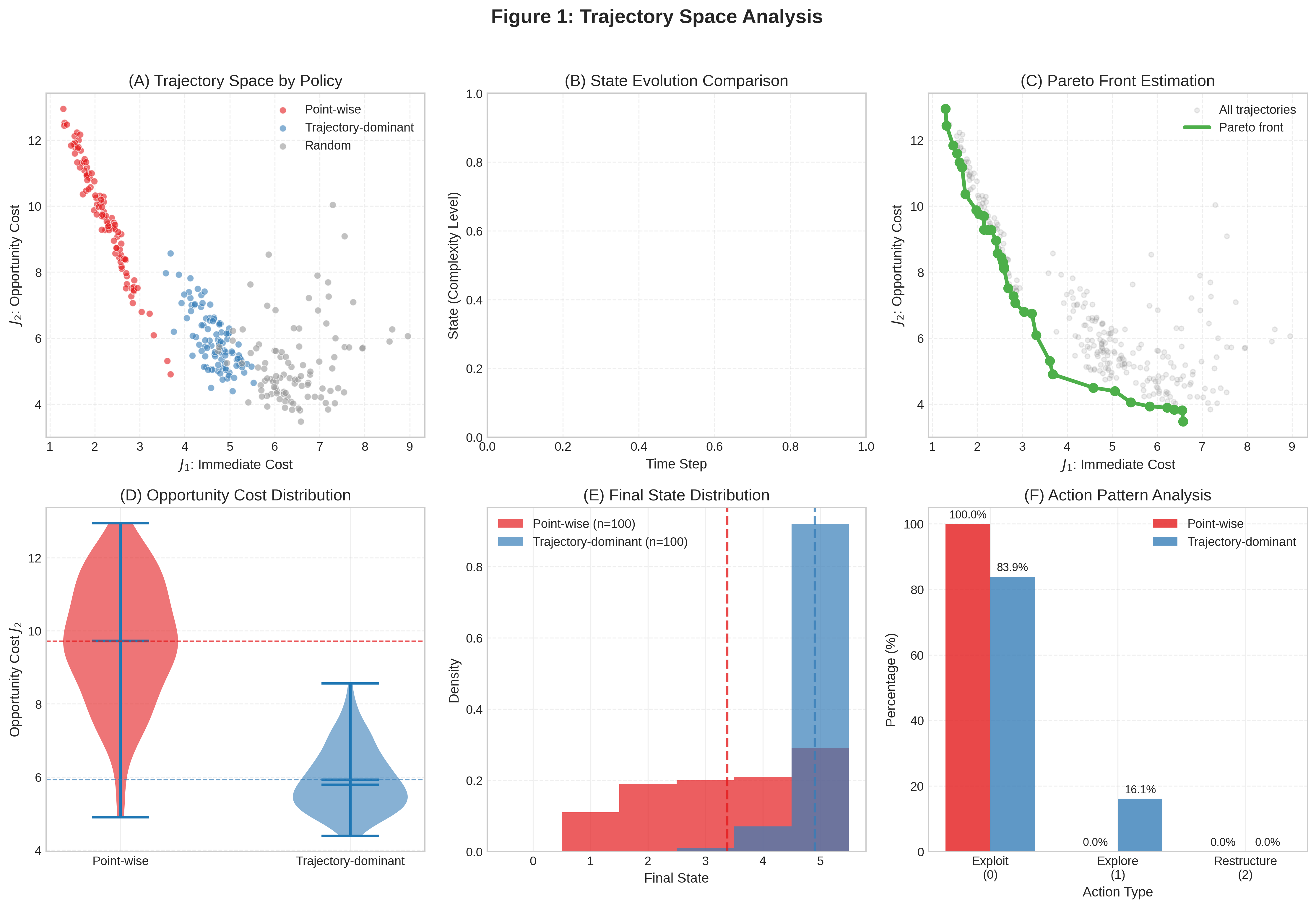}
    \caption{Trajectory Space Analysis. (A) Trajectory distribution by policy. (B) State evolution comparison. (C) Estimated Pareto front. (D) Opportunity cost distribution. (E) Final state distribution. (F) Action pattern analysis.}
    \label{fig:trajectory_analysis}
\end{figure*}

\subsection{Environment and Objectives}
The environment consists of a finite, discrete state space \(X = \{0,1,\ldots,N\}\) representing increasing levels of structural or developmental complexity. The agent begins at an initial state \(x_0 = 0\) and interacts with the environment over a finite horizon \(T\).

At each time step, the agent selects an action from a small action set that induces different trade-offs between short-term efficiency and long-horizon reach. Transitions are deterministic or stochastic but identical across all agents. Crucially, no agent has access to additional information, capacity, or privileged dynamics.

We consider two accumulated objectives: (i) an immediate cost \(J_1\), representing effort, inefficiency, or local disruption, and (ii) an opportunity cost \(J_2\), representing foregone access to higher-complexity states. Both objectives are minimized, yielding a multi-objective trajectory optimization problem.

\subsection{Policy Classes}
We contrast three classes of policies operating under identical dynamics.

\textbf{Point-wise optimizing policies.} These policies select actions that minimize instantaneous or short-horizon cost combinations at each step. While locally efficient, such policies are conservative with respect to trajectory-level dominance, as any deviation that temporarily increases \(J_1\) or \(J_2\) is rejected.

\textbf{Trajectory-dominant policies.} These policies permit temporary increases in accumulated cost in order to preserve access to globally superior trajectories. In particular, they tolerate short-term degradation when doing so enables transitions to higher-complexity regions of the state space.

\textbf{Random policies.} Random action selection provides a baseline illustrating that escape from Pareto traps is not guaranteed by stochasticity alone, but depends on the structure of dominance constraints.

Importantly, the distinction between policy classes lies entirely in their dominance criteria, not in their access to information or computational resources.

\subsection{Trajectory Space and Pareto Traps}
Figure \ref{fig:trajectory_analysis} visualizes the resulting trajectory space in the three policy classes. Panel (A) shows the distribution of accumulated cost vectors \((J_1,J_2)\), revealing that point-wise optimizing agents cluster in regions of low immediate cost but high opportunity cost. These regions are locally non-dominated under trajectory-level comparison, yet globally dominated by trajectories reachable only through temporary degradation.

Panel (C) highlights the Pareto front at the induced trajectory-level. The front is non-convex and its disconnected structure creates the geometric conditions necessary for Pareto traps, consistent with the theoretical analysis in Section 4. Agents confined to a single component of this front cannot reach globally superior components without violating local dominance.

\subsection{Developmental Dynamics and Ceiling Formation}
Panel (B) of Figure \ref{fig:trajectory_analysis} contrasts representative state trajectories. Point-wise optimizing agents rapidly converge to low-complexity states and remain confined there, despite continued optimization. In contrast, trajectory-dominant agents accept early increases in accumulated cost and subsequently access higher-complexity regimes.

This divergence manifests as a dynamic intelligence ceiling. As shown in Panel (E), the final-state distribution of point-wise agents concentrates below a clear threshold, while trajectory-dominant agents consistently exceed this bound. The ceiling is not imposed externally, but emerges endogenously from the geometry of trajectory dominance.

\subsection{Behavioral Signatures of Escape}
Panel (F) of Figure \ref{fig:trajectory_analysis} reveals another structural distinction between the policy classes. Point-wise agents overwhelmingly select exploitative actions, reinforcing confinement within a Pareto trap. Trajectory-dominant agents exhibit a non-zero frequency of exploratory or restructuring actions, reflecting a deliberate violation of local dominance constraints.

These behavioral patterns explain why continued optimization under point-wise criteria deepens local performance without expanding the reach of development. Escaping a Pareto trap requires accepting temporary inefficiency, not additional optimization within the same dominance regime.

\subsection{Interpretation}
This minimal model confirms the central claim of the paper: dynamic intelligence ceilings arise as geometric consequences of trajectory-level dominance, independent of architectural limits, learning algorithms, or computational scale. The model instantiates, in its simplest possible form, the mechanisms formalized in Sections 3 and 4.

The results also clarify why algorithmic remedies, such as exploration bonuses or stochasticity, may not overcome stagnation unless they explicitly relax conservative dominance constraints, a point echoed in prior discussions of exploration, novelty, and open-ended learning \cite{lehman2011abandoning,mouret2015illuminating,wang2020enhanced}.

\section{Discussion}
The trajectory-dominant Pareto framework reframes long-horizon intelligence limits as structural rather than algorithmic phenomena. Instead of attributing stagnation to insufficient learning, data scarcity, or architectural constraints, the analysis presented in this work identifies a geometric origin: optimization confined to locally non-dominated regions of trajectory space.

\subsection{Why Continued Optimization Can Lead to Stagnation}
A central implication of the framework is that local improvement and global adaptability are not equivalent. Under trajectory-level dominance, agents may continue to optimize efficiently within a Pareto trap, refining behavior and reducing accumulated costs along admissible directions. However, such improvements do not expand the set of trajectories accessible to the agent.

This distinction clarifies why intelligent systems may exhibit sustained gains on familiar tasks while failing to access qualitatively new solution regimes. From the perspective of trajectory geometry, optimization effort deepens exploration within a single connected component of the trajectory-Pareto set, without enabling transitions to superior components. Dynamic intelligence ceilings thus arise not as hard bounds, but as emergent limits induced by conservative dominance criteria.

This view complements empirical observations that increasing scale or optimization pressure often yields diminishing returns in adaptability, even as short-term performance continues to improve \cite{hernandez2017measure,sutton2019bitter}. The present framework provides a structural explanation for such effects, independent of specific learning rules or model classes.

\subsection{Relation to Exploration, Novelty, and Open-Ended Learning}
The necessity of violating local dominance in order to escape Pareto traps connects naturally to prior work on exploration and open-ended learning. Approaches such as novelty search and quality-diversity methods deliberately relax objective-driven optimization, allowing agents to traverse regions of the search space that would otherwise be rejected as locally suboptimal \cite{lehman2011abandoning,mouret2015illuminating}.

From the perspective developed here, such methods can be understood as implicitly reducing the rigidity of Pareto traps, either by lowering behavioral inertia or by permitting transitions that temporarily increase accumulated cost. However, existing approaches typically implement these ideas algorithmically, without an explicit characterization of the underlying geometric constraints.

By contrast, the trajectory-dominant Pareto framework provides a definition-level account of why such violations are necessary. It identifies the structural conditions under which exploration must incur temporary degradation in order to yield long-horizon gains, and clarifies why stochasticity or exploration bonuses alone may be insufficient unless they explicitly relax conservative dominance constraints.

\subsection{TEDI as a Structural Diagnostic}
The introduction of the Trap Escape Difficulty Index (TEDI) enables a quantitative comparison of dynamic intelligence ceilings across systems, tasks, and optimization regimes. Rather than measuring performance or regret, TEDI characterizes the rigidity of the trajectory space itself, capturing how distance, structural constraints, and behavioral inertia jointly shape escape difficulty.

This perspective suggests a shift in how intelligence limits might be diagnosed and addressed. Instead of asking whether a system is under-optimized, one may ask whether it is confined to a high-TEDI region of trajectory space. In such cases, further optimization is unlikely to yield qualitative gains unless the dominance structure governing trajectory selection is altered.

Importantly, TEDI is not tied to a specific learning algorithm or architecture. It is an intrinsic property induced by dominance relations over trajectories, and therefore applies equally to symbolic planners, reinforcement learning agents, and hybrid systems.

\subsection{Implications for Scaling and Collapse Phenomena}
Recent work has highlighted failure modes in which intelligent systems lose diversity and adaptability as optimization proceeds, including collapse-like dynamics under recursive training or reinforcement \cite{shumailov2023curse,khanh2025entropy}. The present framework offers a complementary interpretation of these phenomena.

From a trajectory-level perspective, such collapse can be viewed as progressive confinement to increasingly rigid Pareto traps. As optimization intensifies, behavioral diversity contracts, structural constraints harden, and escape difficulty increases. Dynamic intelligence ceilings stabilize not because learning has failed, but because the geometry of accessible trajectories has narrowed.

This interpretation suggests that mitigating collapse requires mechanisms that actively preserve or restore trajectory accessibility, rather than simply increasing scale or optimization pressure. Relaxing local dominance constraints, introducing controlled restructuring phases, or enabling commitment-aware deviations may be necessary to maintain long-horizon adaptability.

\subsection{Scope and Limitations}
The present work is intentionally minimal and structural. It does not propose new algorithms, benchmarks, or large-scale empirical evaluations. While the illustrative model demonstrates the existence and consequences of Pareto traps, it is not intended to capture the full complexity of real-world intelligent systems.

Future work will be required to instantiate these ideas in richer environments, to explore how TEDI can be estimated or bounded in practice, and to examine how dominance criteria interact with learning dynamics in large-scale systems. Nonetheless, the framework provides a unifying lens through which diverse stagnation phenomena can be understood as manifestations of a common geometric constraint.

\section{Scope, Limitations, and Future Directions}
This section clarifies the interpretational scope of the proposed framework, delineates its methodological limits, and outlines directions for future research. The goal is not to enumerate shortcomings, but to precisely locate the contribution within the broader landscape of intelligence research and to indicate how the framework may be extended.

\subsection{Interpretational Scope}
The trajectory-dominant Pareto framework is intended to characterize a specific class of intelligence limitations: those arising from dominance relations defined over accumulated costs along trajectories. As such, its scope is inherently structural. The framework applies to settings in which agent behavior unfolds over extended horizons, where multiple competing objectives shape long-term development, and where conservative dominance criteria govern trajectory selection.

Within this scope, Pareto traps and dynamic intelligence ceilings emerge as geometric properties of trajectory space, independent of particular learning algorithms, representations, or architectures. The framework is therefore agnostic to whether agents are implemented via symbolic planning, reinforcement learning, hybrid systems, or other paradigms.

At the same time, the framework is not intended to replace task-specific evaluation metrics, regret minimization, or benchmark-based assessment. Rather, it complements these approaches by addressing a logically prior question: why locally optimal optimization may systematically restrict access to alternative developmental paths over time. Not all forms of stagnation or performance saturation should be attributed to Pareto traps, and the presence of a ceiling in practice does not by itself imply trajectory-level confinement.

\subsection{Level of Abstraction and Methodological Limits}
The analysis in this work is deliberately conducted at a high level of abstraction. We focus on definition-level properties of trajectory dominance and the induced geometry of feasible trajectories, rather than on algorithmic instantiations or empirical estimators.

In particular, while the Trap Escape Difficulty Index (TEDI) provides a principled quantitative characterization of escape rigidity, the present work does not propose a unique or optimal procedure for estimating TEDI in large-scale or continuous systems. Different approximations, relaxations, or proxies may be appropriate depending on the domain and representation, and exploring such estimators is beyond the scope of this paper.

Similarly, the minimal illustrative model is not intended as a benchmark or as a realistic simulation of complex intelligent behavior. Its role is to demonstrate existence and mechanism: to show that Pareto traps and dynamic intelligence ceilings arise purely from trajectory-level dominance constraints, even under identical dynamics and minimal assumptions. Extending the analysis to richer environments, stochastic transitions, partial observability, or infinite-horizon settings remains an open direction.

\subsection{Relation to Algorithmic Remedies}
Many existing approaches seek to address stagnation through algorithmic interventions, including exploration bonuses, novelty-driven objectives, curriculum learning, or open-ended environment generation. From the perspective developed here, such methods can be understood as implicitly modifying the dominance structure governing trajectory selection, either by relaxing conservative criteria or by introducing mechanisms that permit temporary degradation.

However, the present framework does not prescribe any specific remedy. Instead, it provides a diagnostic lens for understanding when and why algorithmic interventions are necessary. In settings characterized by high TEDI, further optimization within the same dominance regime is unlikely to yield qualitative gains, regardless of scale or training duration. In contrast, low-TEDI regions may admit escape through modest restructuring or controlled violations of local dominance.

Recognizing this distinction may help explain why some exploration-based methods succeed in certain regimes but fail in others, and why stochasticity alone is often insufficient to overcome deep developmental constraints.

\subsection{Future Directions}
The trajectory-dominant Pareto framework opens several directions for future research. On the theoretical side, extending the analysis to infinite-horizon settings, stochastic dominance relations, and partially observable environments would broaden its applicability. Formal connections between TEDI and other geometric or information-theoretic quantities may also yield deeper insights into the structure of developmental constraints.

On the system level, the framework suggests the importance of mechanisms that explicitly manage trajectory accessibility. These include adaptive restructuring phases, commitment-aware deviations, and policies that reason about temporary degradation as a strategic resource rather than a failure mode. Investigating how such mechanisms can be implemented in learning systems without sacrificing stability remains a key challenge.

Finally, the framework invites a re-examination of intelligence growth and collapse phenomena as manifestations of trajectory-space geometry. Understanding how Pareto traps form, harden, and dissolve over time may provide new perspectives on long-horizon adaptability, scaling limits, and the design of intelligent systems capable of sustained developmental expansion.

\section{Conclusion}
This work has proposed a foundational reformulation of intelligence as a trajectory-level optimization problem governed by Pareto dominance. By shifting attention from point-wise performance and terminal outcomes to the geometry of accessible trajectories, we have shown how dynamic intelligence ceilings arise as structural consequences of locally conservative dominance relations.

Within this framework, Pareto traps emerge as locally non-dominated yet globally restrictive regions of trajectory space. Agents confined to such regions may continue to optimize efficiently while remaining unable to access superior developmental paths. The resulting intelligence ceilings are neither absolute capability limits nor failures of learning, but geometric constraints induced by trajectory-level dominance.

To characterize these constraints, we introduced a mathematical taxonomy of Pareto traps and proposed the Trap Escape Difficulty Index (TEDI) as an intrinsic measure of escape rigidity. Together, these tools provide a principled basis for comparing developmental limits across systems and for distinguishing ceilings that can be overcome through modest restructuring from those requiring fundamental changes in trajectory dominance.

A minimal illustrative model demonstrated that these phenomena arise independently of architectural scale or algorithmic sophistication. Escaping dynamic intelligence ceilings requires deliberate violations of local dominance constraints, not merely increased optimization within a fixed regime.

By reframing intelligence as a property of trajectory accessibility, this work offers a structural explanation for long-horizon stagnation and collapse-like dynamics observed in contemporary artificial systems. More broadly, it lays the groundwork for future research on adaptive restructuring, developmental phase transitions, and commitment-aware forms of artificial intelligence capable of sustained growth beyond locally optimal regimes.

\section*{Acknowledgements}
The authors thank colleagues and collaborators for insightful discussions that helped clarify the conceptual framework presented in this work. No external funding was received for this research.

\section*{Author Contributions}
\textbf{Conceptualization:} T.X.K., T.Q.H.; \textbf{Formal analysis:} T.X.K.; \textbf{Methodology:} T.X.K., T.Q.H.; \textbf{Writing—original draft:} T.X.K.; \textbf{Writing—review and editing:} T.X.K., T.Q.H.

\section*{Funding}
This research received no external funding.

\section*{Conflicts of Interest}
The authors declare no conflict of interest.

\section*{Data Availability}
Data sets were not generated or analyzed during the current study.

\appendix
\section{Supplementary Information}

\subsection{Formal Definitions and Preliminary Properties}
We work in a finite-horizon Markov Decision Process (MDP) with state space \(\mathcal{X}\) (finite), action space \(\mathcal{A}\) (finite), transition kernel \(P\), initial state \(x_0 \in \mathcal{X}\), and horizon \(T < \infty\). A trajectory is a sequence
\[
\tau = (x_0,a_0,x_1,a_1,\ldots ,x_{T-1},a_{T-1},x_T),
\]
where \(x_{t+1} \sim P(\cdot | x_t, a_t)\). For \(m\) objectives, the accumulated cost of objective \(i\) along \(\tau\) is
\[
J_i(\tau) = \sum_{t=0}^{T-1} \ell_i(x_t,a_t,t), \quad i = 1,\ldots,m,
\]
with cost vector \(\mathbf{J}(\tau) = (J_1(\tau), \ldots, J_m(\tau)) \in \mathbb{R}^m\). Let \(\mathcal{T}\) denote the (finite) set of all feasible trajectories starting from \(x_0\).

\begin{definition}[Trajectory Dominance (strict)]
A trajectory \(\tau\) trajectory-dominates \(\tau'\) (denoted \(\tau \succ \tau'\)) if
\[
J_i(\tau) \leq J_i(\tau') \quad \forall i \in \{1,\ldots,m\} \quad \text{and} \quad \exists j: J_j(\tau) < J_j(\tau').
\]
\end{definition}

\begin{definition}[Trajectory-Pareto Optimality]
A trajectory \(\tau^* \in \mathcal{T}\) is trajectory-Pareto optimal if there exists no \(\tau \in \mathcal{T}\) such that \(\tau \succ \tau^*\).
\end{definition}

\begin{lemma}[Basic properties of \(\succ\)]
The relation \(\succ\) is:
\begin{enumerate}
    \item Irreflexive: \(\tau \not\succ \tau\) for all \(\tau\).
    \item Asymmetric: \(\tau \succ \tau' \implies \tau' \not\succ \tau\).
    \item Transitive: \(\tau_1 \succ \tau_2 \succ \tau_3 \implies \tau_1 \succ \tau_3\).
\end{enumerate}
\end{lemma}

\begin{proof}
Irreflexivity and asymmetry follow immediately from the definition (strict inequality in at least one component). Transitivity follows from the component-wise comparison of real numbers under the product order \(\leq\) with at least one strict.
\end{proof}

\begin{lemma}[Existence of trajectory-Pareto optimal trajectories]
The set \(\mathcal{T}^*\) of trajectory-Pareto optimal trajectories is non-empty.
\end{lemma}

\begin{proof}
\(\mathcal{T}\) is finite \((|\mathcal{X}|^T |\mathcal{A}|^T < \infty)\). Consider the partial order \(\preceq\) on \(\mathbb{R}^{m}\) induced by \(\mathbf{J}(\tau) \preceq \mathbf{J}(\tau')\) iff \(J_i(\tau) \leq J_i(\tau') \forall i\). The strict dominance \(\succ\) corresponds to \(\preceq\) with at least one strict inequality. By finiteness, every chain has a minimal element (Kuratowski's lemma or direct induction on \(|\mathcal{T}|\)). Any minimal element w.r.t. \(\preceq\) cannot be strictly dominated, hence belongs to \(\mathcal{T}^*\).
\end{proof}

\subsubsection{Extension to Infinite Horizon}
For infinite-horizon MDPs with discount factor \(\gamma < 1\), accumulated costs become \(J_i(\tau) = \sum_{t=0}^{\infty} \gamma^t \ell_i(x_t,a_t,t)\). Assuming bounded costs \((|\ell_i| \leq M)\), the space is compact under weak topology. Existence follows from Zorn's lemma applied to the partially ordered set of cost vectors, as every chain has a lower bound by compactness.

\subsection{Pareto Traps and Dynamic Intelligence Ceilings}
\begin{definition}[Local trajectory neighborhood]
For \(\tau \in \mathcal{T}\), let \(N_\epsilon(\tau)\) be the set of trajectories obtainable by changing at most \(\epsilon\) actions in \(\tau\) (Hamming distance \(\leq \epsilon\), \(\epsilon \in \mathbb{N}\) small).
\end{definition}

\begin{definition}[Pareto trap (formal)]
A non-empty set \(S \subseteq \mathcal{T}\) is a Pareto trap if:
\begin{enumerate}
    \item \(S\) is locally trajectory-Pareto optimal: \(\forall \tau \in S\), \(\nexists \tau' \in N_\epsilon(\tau)\) such that \(\tau' \succ \tau\).
    \item \(S\) is globally restrictive: \(\exists \hat{\tau} \notin S\) such that \(\hat{\tau} \succ \tau\) for some (hence all, by transitivity in connected components) \(\tau \in S\), and reaching \(\hat{\tau}\) requires a trajectory \(\sigma\) with \(\sigma \not\succ \tau\) for some intermediate step (temporary degradation).
\end{enumerate}
\end{definition}

\begin{lemma}[Existence of Pareto traps]
Under mild conditions (non-convex trajectory cost landscape, i.e., \(\mathbf{J}(\mathcal{T})\) is not convex), Pareto traps exist.
\end{lemma}

\begin{proof}
Consider the image \(\mathbf{J}(\mathcal{T}) \subseteq \mathbb{R}^m\). Let \(\mathcal{P} = \{\mathbf{J}(\tau) \mid \tau \in \mathcal{T}^*\}\) be the trajectory-Pareto front (non-dominated points). If \(\mathcal{P}\) is disconnected (possible when the feasible set induces non-convexity, e.g., via high-cost "bridges"), then any connected component \(C \subset \mathcal{P}\) that is not the globally minimal component forms a trap: trajectories mapping to \(C\) are locally non-dominated within small perturbations (by continuity of \(\mathbf{J}\) in finite spaces), yet globally dominated by trajectories in superior components. Reaching superior components requires crossing a "valley" of temporary degradation (increase in at least one \(J_i\)), violating local dominance.
\end{proof}

\begin{definition}[Dynamic intelligence ceiling (formal)]
Let \(\mathcal{S}\) be the Pareto trap containing the agent's current trajectories. The dynamic intelligence ceiling is
\[
C = \sup_{\tau \in \mathcal{S}} f(\mathbf{J}(\tau)),
\]
where \(f:\mathbb{R}^{m} \to \mathbb{R}\) is a scalarization of effective intelligence (e.g., \(f(\mathbf{v}) = -\sum w_i v_i\) or a monotonic transformation of developmental reach).
\end{definition}

Continued optimization within a trap only refines \(C\) locally; global ceiling lift requires deliberate violation of local dominance (restructuring mechanisms).

\subsubsection{Conditions Under Which Pareto Traps Do Not Arise}
Pareto traps do not arise if any of the following holds:
\begin{itemize}
    \item the trajectory space \(\mathcal{T}\) is convex under admissible transitions,
    \item dominance is evaluated over full-horizon trajectories without local constraints,
    \item exploration permits arbitrary stochastic transitions without cost barriers.
\end{itemize}
These conditions clarify that Pareto traps are not universal, but structural.

\subsection{Minimal Illustrative Model: Explicit Construction and Proof}
Define the environment as follows:
\begin{itemize}
    \item State space: \(\mathcal{X} = \{0,1,\ldots,N\}\) (complexity levels).
    \item Actions: \(\mathcal{A} = \{\text{Refine, Advance, Restructure}\}\).
    \item Transitions:
    \[
    \begin{cases}
        \text{Refine}: & x \to x \\
        \text{Advance}: & x \to \min(x+1, N) \\
        \text{Restructure}: & x \to \min(x+2, N) \text{ (with probability } p\text{)}
    \end{cases}
    \]
    \item Costs (two objectives: efficiency \(J_1\), creativity/reach \(J_2\)):
    \[
    \begin{cases}
        \text{Refine}: & (\ell_1=0.1, \ell_2=1.0/(x+1)) \\
        \text{Advance}: & (\ell_1=0.8, \ell_2=0.5+0.3(5-x)) \\
        \text{Restructure}: & (\ell_1=2.0, \ell_2=1.0+0.4(5-x))
    \end{cases}
    \]
\end{itemize}

Point-wise agent: Greedy on instantaneous cost \(\Rightarrow\) always Refine \(\Rightarrow\) stays at low \(x\), high \(J_1\) but low reach.

Trajectory-dominant agent: Permits temporary spikes in cost \(\Rightarrow\) occasional Restructure \(\Rightarrow\) reaches \(x = N\).

\begin{lemma}[Confinement in minimal model]
The point-wise agent is confined to a Pareto trap \(S = \{\tau: x_t \leq k < N \ \forall t\}\) while a trajectory-dominant agent escapes.
\end{lemma}

\begin{proof}
For point-wise: Any deviation (Advance/Restructure) increases immediate \(J_1\) strictly \(\Rightarrow\) dominated locally. Global superior trajectories require temporary \(J_1\) spike \(\Rightarrow\) trap. For trajectory-dominant: By allowing \(\epsilon\)-degradation, it accesses restructuring paths whose total \(\mathbf{J}\) dominates all trap trajectories (explicit computation: restructuring cost 10 once unlocks \(\Delta J_2 = +O(N)\)).
\end{proof}

This explicit model confirms that ceilings arise purely from dominance criteria, independent of architectural limits.

\subsection{Quantitative Toy Model: Extended Description}
This section provides a detailed description of the toy model used to generate Figure 1 in the main paper.

\subsubsection{Environment}
The state space is discrete: \(\mathcal{X} = \{0,1,2,3,4,5\}\), representing increasing levels of structural complexity.

The action set is: \(\mathcal{A} = \{\text{Exploit}, \text{Explore}, \text{Restructure}\}\).

Actions differ in immediate cost and transition effect, inducing a trade-off between short-term efficiency and long-horizon opportunity.

\subsubsection{Objectives}
Two objectives are accumulated:
\begin{itemize}
    \item \(J_1\): immediate cost,
    \item \(J_2\): opportunity cost, decreasing with access to higher states.
\end{itemize}
Both objectives are minimized, yielding a multi-objective trajectory evaluation problem.

\subsubsection{Policy Classes}
We compare:
\begin{itemize}
    \item point-wise optimizing policies minimizing instantaneous cost combinations,
    \item trajectory-dominant policies permitting temporary cost increases,
    \item random baseline policies.
\end{itemize}
All policies operate under identical dynamics and horizons.

\subsection{Reproducible Simulation Code}
\begin{lstlisting}[caption={Python implementation of the minimal model},label=code:model]
import numpy as np
import matplotlib.pyplot as plt

class TrajectoryToyModel:
    def __init__(self, n_states=6, horizon=30):
        self.n_states = n_states
        self.horizon = horizon
    
    def transition(self, state, action):
        if action == 0:  # Exploit
            return state
        elif action == 1:  # Explore
            return min(state + 1, self.n_states - 1)
        elif action == 2:  # Restructure
            if state < self.n_states - 1:
                return np.random.randint(state + 1, 
                         self.n_states)
            return state
    
    def immediate_cost(self, state, action):
        if action == 0:  # Exploit
            return 0.1 + 0.05*state, 1.0/(state+1)
        elif action == 1:  # Explore
            return 0.8, 0.5 + 0.3*(5-state)
        elif action == 2:  # Restructure
            return 2.0, 1.0 + 0.4*(5-state)
    
    def run(self, policy):
        state = 0
        J1, J2 = 0.0, 0.0
        states = [state]
        for t in range(self.horizon):
            if policy == "pointwise":
                costs = [sum(self.immediate_cost(
                           state, a)) for a in range(3)]
                action = np.argmin(costs)
            elif policy == "trajectory":
                if np.random.random() < 0.1:  
                    action = 2  # 10% restructure
                elif state < 3:  
                    action = 1  # Explore if low state
                else:
                    action = 0
            else:
                action = np.random.choice([0,1,2])
            
            c1, c2 = self.immediate_cost(state, action)
            J1 += c1
            J2 += c2
            state = self.transition(state, action)
            states.append(state)
        return states, J1, J2

# Example simulation
model = TrajectoryToyModel()
states_pw, J1_pw, J2_pw = model.run("pointwise")
states_td, J1_td, J2_td = model.run("trajectory")
states_rand, J1_rand, J2_rand = model.run("random")

# Plot trajectories
fig, ax = plt.subplots(figsize=(10,6))
ax.plot(states_pw, label='Point-wise')
ax.plot(states_td, label='Trajectory-dominant')
ax.plot(states_rand, label='Random')
ax.set_xlabel('Time Step')
ax.set_ylabel('State (Complexity Level)')
ax.set_title('Trajectories Under Different Policies')
ax.legend()
plt.savefig('trajectory_plot.png', dpi=300, 
            bbox_inches='tight')

# Plot Pareto front
fig, ax = plt.subplots(figsize=(8,6))
ax.scatter(J1_pw, J2_pw, label='Point-wise', marker='o')
ax.scatter(J1_td, J2_td, label='Trajectory-dominant', 
           marker='x')
ax.scatter(J1_rand, J2_rand, label='Random', marker='s')
ax.set_xlabel('J1 (Immediate Cost)')
ax.set_ylabel('J2 (Opportunity Cost)')
ax.set_title('Cost Vectors for Policies')
ax.legend()
plt.savefig('pareto_front_plot.png', dpi=300, 
            bbox_inches='tight')
\end{lstlisting}

The code may be executed with standard Python scientific libraries. Typical outputs show point-wise agents stuck at low states (e.g., state 0), while trajectory-dominant agents reach higher states (e.g., up to 5), confirming trap confinement.

\begin{figure*}[t]
    \centering
    \includegraphics[width=0.95\textwidth]{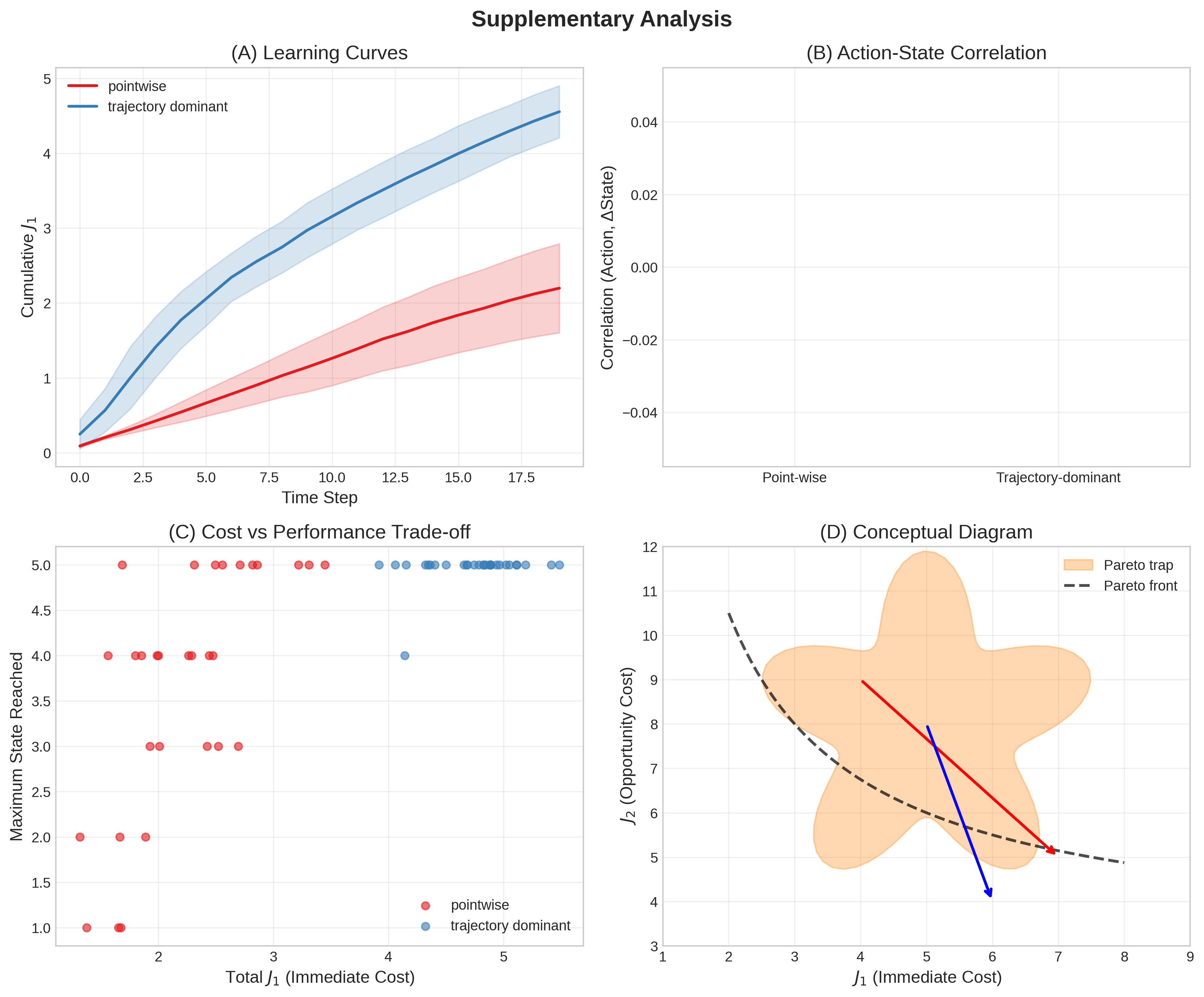}
    \caption{Supplementary Analysis. (A) Learning curves for cumulative $J_1$. (B) Action-state correlation. (C) Cost vs. performance trade-off. (D) Conceptual diagram of Pareto trap and escape path.}
    \label{fig:supplementary_analysis}
\end{figure*}

\end{document}